\documentclass[twocolumn]{autart}    % Enable this line and disable the 
\usepackage{graphicx}          % Include this line if your 
\usepackage{graphics} % for pdf, bitmapped graphics files
\usepackage{amsmath} % assumes amsmath package installed
\usepackage{amssymb}  % assumes amsmath package installed
\usepackage{mathtools}
\usepackage{color}
\usepackage{cite}

\newtheorem{theorem}{Theorem}
\newtheorem{proposition}[theorem]{Proposition}

\newtheorem{corollary}{Corollary}

\newtheorem{definition}{Definition}

\makeatletter
\DeclareRobustCommand{\qed}{%
  \ifmmode % if math mode, assume display: omit penalty etc.
  \else \leavevmode\unskip\penalty9999 \hbox{}\nobreak\hfill
  \fi
  \quad\hbox{\qedsymbol}}
\newcommand{\openbox}{\leavevmode
  \hbox to.77778em{%
  \hfil\vrule
  \vbox to.675em{\hrule width.6em\vfil\hrule}%
  \vrule\hfil}}
\newcommand{\qedsymbol}{\openbox}
\newenvironment{proof}[1][\proofname]{\par
  \normalfont
  \topsep6\p@\@plus6\p@ \trivlist
  \item[\hskip\labelsep\itshape
    #1.]\ignorespaces
}{%
  \qed\endtrivlist
}
\newcommand{\proofname}{Proof}
\makeatother

\begin{document}

\begin{frontmatter}

\title{Efficient Trajectory Generation for Robotic Systems Constrained by Contact Forces} % Title, preferably not more 

% \title{Efficient Trajectory Generation for Robotic Systems \\ Constrained by Contact Force via  Nonlinear Programming} % Title, preferably not more 

\thanks[footnoteinfo]{Some part of this paper will be presented at American Control Conference 2019 \cite{lee2018trajectory}. Corresponding author L. Sentis.}

\author[ME]{Jaemin Lee}\ead{jmlee87@utexas.edu},    % Add the 
\author[ASE]{Efstathios Bakolas}\ead{bakolas@austin.utexas.edu},               % e-mail address 
\author[ASE]{Luis Sentis}\ead{lsentis@austin.utexas.edu}  % (ead) as shown

\address[ME]{Department of Mechanical Engineering, The University of Texas at Austin, TX, 78712-1221, USA}  % Please supply                                              
\address[ASE]{Department of Aerospace Engineering and Engineering Mechanics, The University of Texas at Austin, TX 78712-1221, USA}             % full addresses

\begin{keyword}                           % Five to ten keywords,  
Robots; Reachability; Contact Force.               % chosen from the IFAC 
\end{keyword}                             % keyword list or with the 
                                          % help of the Automatica 
                                          % keyword wizard

\begin{abstract}                          % Abstract of not more than 200 words.
In this work, we propose a trajectory generation method for robotic systems with contact force constraint based on optimal control and reachability analysis. Normally, the dynamics and constraints of the contact-constrained robot are nonlinear and coupled to each other. Instead of linearizing the model and constraints, we directly solve the optimal control problem to obtain the feasible state trajectory and the control input of the system. A tractable optimal control problem is formulated which is addressed by dual approaches, which are sampling-based dynamic programming and rigorous reachability analysis. The sampling-based method and Partially Observable Markov Decision Process (POMDP) are used to break down the end-to-end trajectory generation problem via sample-wise optimization in terms of given conditions. The result generates sequential pairs of subregions to be passed to reach the final goal. The reachability analysis ensures that we will find at least one trajectory starting from a given initial state and going through a sequence of subregions. The distinctive contributions of our method are to enable handling the intricate contact constraint coupled with system's dynamics due to the reduction of computational complexity of the algorithm. We validate our method using extensive numerical simulations with a legged robot.
\end{abstract}

\end{frontmatter}

\section{Introduction}
%% Simple introduction of our problem
This paper considers the optimal control of robotic systems with contact force constraints. Often, it is required that legged or humanoid robots maintain stable foot or body contacts while executing given tasks. In such cases, contact forces constrain and determine the robot's state reachability together with other state and input constraints. Therefore, we seek to devise control algorithms that can generate trajectories for contact-constrained robots via formal state reachability analysis. Often, control studies for robotics assume that task trajectories are predefined \cite{khatib1987unified, righetti2011inverse, sentis2005synthesis}, then attempt to find an instantaneously optimal solution to accomplish them. However the desired trajectories are frequently infeasible and it is not straight-forward to check the feasibility of trajectories under contact constraints a priori. Many motion planning and trajectory generation approaches for humanoid robots use very simple models of a robot such as considering center of mass dynamics under contact constraints \cite{kajita2003biped,stephens2010dynamic,liu2015trajectory}. However, those methods result on lower performance of the robots since they cannot capture the robot's kinematics or input constraints among other limitations. 

%% Optimal Control (iLQR) but exist limits so we choose NLP
Optimal control is an alternative approach to solve the trajectory generation problem for contact-constrained robots. Recently, trajectory generation for the legged robots was formulated as a bi-level optimization problem and solved by an iterative Linear Quadratic Regulator (iLQR) \cite{carius2018trajectory}. However, the iLQR has still challenging issues to address generic nonlinear constraints without constraint softening. This could result on motion planers that violate important physical constraints of robots. To strictly consider the constraints and nonlinearity of the robot models, we should directly solve the optimal control problem via Nonlinear Programming (NLP) in the process of obtaining feasible trajectories. Although many NLP solvers, i.e. SNOPT \cite{gill2005snopt} and IPOPT \cite{wachter2006implementation}, are available, NLP has significantly challenging issues. One is high computational cost for obtaining the solution. Also, feasible initial conditions are necessary for NLP. In our work, we propose two complementary processes to resolve these problems: reachability analysis and sampling-based dynamic programming. 

%% reachability analysis systems and concept (including HJB PDE)
The reachability problem consists of checking whether the state of the system can reach a specific state over a finite time horizon, starting from a given initial state. The field of robotics has often focused on configuration space reachability guided by the given tasks such as selecting the stance location of humanoid robots \cite{burget2015stance, yang2017efficient}. Although these methods are useful for kinematic feasibility, they are limited to address the requirements of dynamical systems considering various constraints such as joint velocity/torque limits and contact force constraints. To address this gap, we will employ optimal control on the nonlinear dynamical system with constraints. 

In optimal control, the reachability analysis has been often used for nonlinear systems \cite{althoff2013reachability,rungger2018accurate,scott2013bounds,althoff2008reachability,althoff2014reachability}, hybrid dynamical systems \cite{mitchell2001validating,habets2006reachability,mitchell2005time,maiga2016comprehensive}, and stochastic systems \cite{summers2013stochastic,lesser2014reachability,abate2008probabilistic}. We can categorize established reachability analysis methods for nonlinear systems into three groups: 1) solving Hamilton-Jacobi-Bellman PDE, 2) using linearization and mathematical approximation 3) using propagation and mappings of a set of reachable states. First, for low dimensional dynamical systems, the reachability analysis is often achieved via Hamilton-Jacobi-Bellman PDE \cite{asarin2000approximate,kariotoglou2013approximate}. For some systems, it is impossible to perform the reachability analysis by solving Hamilton-Jacobin-Bellman PDE. Instead, many approaches have been proposed to compute reachable sets exploiting mathematical techniques, optimization, inherent characteristics of systems, etc.     

%% detailed reachability analysis methods (excluding HJB PDE)
Other than the methods using Hamilton-Jacobi-Bellman PDE, many methodologies have been proposed to obtain reachable sets of systems. The logarithmic norm of a type of system's Jacobian is utilized to obtain over-approximated reachable sets for nonlinear continuous-time systems \cite{maidens2015reachability} and that norm is utilized for simulation-based reachability analysis \cite{arcak2017simulation}. Another approach tries to do more accurate reachability analysis for uncertain nonlinear systems by using more conservative approximations \cite{althoff2008reachability, rungger2018accurate}. Also, for continuous-time piecewise affine systems, linear matrix inequalities (LMI) are employed to characterize the bounds of reachable regions \cite{hamadeh2008reachability}. Another class of  reachability analysis uses convex sets for approximation such as ellipsoid \cite{kurzhanski2000ellipsoidal,kim2008improved}, polytopes, zonotopes \cite{girard2005reachability}, and support functions \cite{girard2008efficient,le2010reachability}. Although those approximation-based approaches are capable of extending to nonlinear systems, they only consider convex sets and often ignore other constraints. These are challenges for extending the previous approaches to more sophisticated and complicated systems. Additionally, computational complexity exponentially increases with respect to the dimension of the state space and the time length for those methods.

%% Our method
Our problem considers a constrained nonlinear system with a constrained variable, e.g. a contact force, coupled with system dynamics. Since the contact force is time varying and our problem is also high dimensional, it is very difficult to do reachability analysis of our system via Hamilton-Jacobi-Bellman PDE. Linearization of dynamics and approximating reachable sets with convex sets is not applicable to our problem because reachable sets of constrained nonlinear systems may not be convex. Moreover, we want to be strictly compared to methods softening or linearizing constraints like iLQR. Thus, we devise a new method consisting of propagating system states and approximating the reachable set. NLP using optimal control utilizes the approximated reachable set. In order to address the increasing computational complexity, we propose various techniques which are described thereafter. And, we do so, in the context of robotic systems with contact force constraints, in a computational efficient way.  

Concretely, we incorporate a sampling-based approach, quadratic programming (QP), NLP, and approximation techniques such as propagation of boundary samples to solve our problem. More specifically, for dividing the end-to-end trajectory generation problem, we obtain the constrained-state set using a sampling-based approach and QP, then, reformulate the problem as a Partially Observable Markov Decision Process (POMDP) in the system's output space. An optimal Markov policy resulting from the dynamic programming (DP) provides a sequence of output subregions. The sequence of output subregions guides the path of output with avoiding unsafe output regions such as locations of obstacles in the output space. In the next step, we implement a rigorous reachability analysis between given pairs of subregions by propagating the states from the given initial state. In addition, we propose a method to approximate the reachable set using propagation of boundary states. The algorithmic efficiency of our method is one of the contributions of our work. 

This paper is organized as follows. Section \ref{sec2} defines our problem and the target class of system. A sampling-based algorithm for obtaining the set of constrained states in Section \ref{sec3} and a POMDP for obtaining an optimal Markov policy are described in Section \ref{sec4}. In Section \ref{sec5}, we propose an approach to obtaining the reachable set and analyzing the method in detail. Based on the result of rigorous reachability analysis in Section \ref{sec5}, an optimal controller is designed by implementing the NLP of Section \ref{sec6}. The proposed approach is validated by simulation of a robotic legged system with contact force constraint in Section \ref{sec6}.   

\section{Problem Formulation}
\label{sec2}

\subsection{Notation}
We denote the set of real $n$-dimensional vectors and the set of real $n\times m$ matrices by $\mathbb{R}^{n}$ and $\mathbb{R}^{n \times m}$, respectively. The sets of non-negative and non-positive real numbers are represented as $\mathbb{R}_{\leq 0}$ and $\mathbb{R}_{\geq 0}$, respectively. The set of natural numbers and the set of integer numbers are denoted by $\mathbb{N}$ and $\mathbb{Z}$, respectively. The set of positive definite $n\times n$ matrices and the set of positive semi-definite $n\times n$ matrices are denoted by $\mathbb{S}_{>0}^{n}$ and $\mathbb{S}_{\geq 0}^{n}$. When considering $z_1, z_2 \in \mathbb{N}$ with $z_2>z_1$, the discrete interval between $z_1$ and $z_2$ is defined as $[z_1, z_2]_{\mathbb{N}} \coloneqq \{z_1, z_1+1, \dots, z_2-1, z_2\}$. In case of real numbers $z_1, z_2 \in \mathbb{R}_{\geq 0}$, $[z_1, z_2]_{d}^{\Delta} \coloneqq \{ z_1, z_1+\Delta ,\dots, z_2-\Delta, z_2\}$ denotes a discrete interval with $\Delta$ being the increment. When $n$ real numbers $a_{1},\dots, a_{n}$ are consider, $\mathrm{Vec}[a_i]_{i=1}^{n} \in \mathbb{R}^{n}$ represents a vector whose $i$-th element is $a_i$. Given $n\times m$ real numbers $a_{11},\dots,a_{mn}$, a matrix whose $(i,j)$ element is $a_{ij}$ is denoted by $\mathrm{Mat}[a_{ij}]_{i,j=1}^{n,m} \in \mathbb{R}^{n\times m}$. Given a square matrix $\mathbf{A}\in \mathbb{R}^{n\times n}$, $\mathrm{tr}(\mathbf{A})$ denotes its trace. $\overline{\sigma}(\mathbf{A})$ and $\underline{\sigma}(\mathbf{A})$ represent the largest and smallest singular values of $\mathbf{A}$, respectively. Given matrices $\mathbf{A}_{i} \in \mathbb{R}^{n_i\times m}$ $i \in [1,z]_{\mathbb{N}}$, $\mathrm{Vertcat}(\mathbf{A}_{1}, \dots, \mathbf{A}_{z}) \in \mathbb{R}^{(n_q +\dots+n_z)\times m}$ indicates a block matrix constructed by vertically concatenating the matrices $\mathbf{A}_{i}$ $i \in [1,z]_{\mathbb{N}}$. Given a set of real vectors $\mathcal{A} \subseteq \mathbb{R}^{n}$, $\mathrm{card}(\mathcal{A})$ denotes its cardinality. When considering particular cases such that $\mathcal{A} \subset \mathbb{R}^n$ with $n \in [1,3]_\mathbb{N}$, $\mathrm{ghull}(\mathcal{A})$ and $\mathrm{gbd}(\mathcal{A})$ represent the general hull and the set of vectors closest the boundary of $\mathcal{A}$. $\mathbb{E}[.]$ represents the probabilistic expectation operator. 

% $\mathcal{O}(z)$, $z\in \mathbb{R}_{\geq0}$, represents the big O notation that characterizes the computational complexity.

\subsection{Nonlinear System Model}
We characterize the equation of motion for general robotic systems with contact forces and assuming rigid body linkages as follows:
\begin{equation}\label{eq:robot_dyn}
    \mathbf{M}(q)\ddot{q} + g(\dot{q}, q) = \mathbf{S}^{\top}u + \mathbf{J}_c^{\top}(q) F_{c} 
\end{equation}
where $q\in \mathbb{R}^{n_q}$, $\mathbf{M}(q)\in \mathbb{S}_{>0}^{n_q}$, $g(\dot{q},q)\in \mathbb{R}^{n_q}$, $\mathbf{S}\in \mathbb{R}^{n_q \times n_u}$, $u \in \mathbb{R}^{n_u}$, $\mathbf{J}_c(q)\in\mathbb{R}^{n_c \times n_q}$, and $F_{c}\in\mathbb{R}^{n_c}$ denote the joint variable, sum of Coriolis/centrifugal and gravitational forces, selection matrix for the actuation, input actuating joint torques, contact Jacobian matrix, and contact force, respectively. We can bring the differential equation (\ref{eq:robot_dyn}) into a state space form by defining the state $x \coloneqq [ x_1^{\top} x_2^{\top} ]^{\top}\in \mathbb{R}^{n_x}$ where $x_1 = q$ and $x_2 = \dot{q}$:
\begin{subequations}
\begin{align}
    \dot{x}(t) &= f_{\mathfrak{C}} (x(t), u(t), F_c(t)) \\
    &= f_x(x(t)) + f_u(x(t))u(t) + f_c(x(t)) F_c(t) \label{eq:state_space}\\
    y(t) &= f_{y}(x(t)) \\
    f_x(x) &\coloneqq \left[\begin{array}{c} x_2(t) \\ -\mathbf{M}^{-1}(x_1)g(x_1,x_2) \end{array} \right]\nonumber  \\
    f_u(x) &\coloneqq \left[\begin{array}{c} \mathbf{0}_{n_q \times n_u} \\ \mathbf{M}^{-1}(x_1)\mathbf{S}^{\top} \end{array} \right], f_{c}(x) \coloneqq \left[ \begin{array}{c} \mathbf{0}_{n_q \times n_c} \\ \mathbf{M}^{-1}(x_1)\mathbf{J}_c^{\top}(x_1) \end{array}\right]\nonumber 
\end{align}
\end{subequations}
where $f_{x}:\mathbb{R}^{n_x}\mapsto \mathbb{R}^{n_x}$, $f_{u}:\mathbb{R}^{n_x} \mapsto \mathbb{R}^{n_u}$, and $f_c:\mathbb{R}^{n_x}\mapsto \mathbb{R}^{n_c}$ are nonlinear functions of the state $x$. $f_y: \mathbb{R}^{n_x} \mapsto \mathbb{R}^{n_y}$ is a nonlinear output function that represents the desired tasks, i.e. forward kinematics of points of robot that we wish to control. We define a discrete time state space model of the contact-constrained robotic system (DTSCR) as the discrete counterpart of the state space model:
\begin{equation}\label{CTSCR}
\begin{split}
    x(t_{k+1}) &= f_{\mathfrak{D}}\left(x(t_{k}),u(t_{k}),F_c(t_{k}) \right) \\
    &= x(t_{k}) + \sum_{i=1}^{\infty} \frac{B_{i}(x(t_k),u(t_k), F_c(t_k))}{i!} (\Delta t)^{i} 
        \end{split}
\end{equation}
where $\Delta t$ denotes the sampling period for the discretization of the model. $f_{\mathfrak{D}}:\mathbb{R}^{n_x+n_u+n_c}\mapsto \mathbb{R}^{n_x}$ and $B_{i}(x,u,F_c) = \frac{\partial B_{i-1}(x,u,F_c)}{\partial x} B_{1}(x,u,F_c)$ and $B_{1}(x,u,F_c) = f_x(x)+f_u(x)u + f_c(x)F_c$.

\subsection{Constraints of the System}
We refer to $h_{e}$ and $h_{i}$ as the equality constraint function and the inequality constraint function, respectively, where we assume that $h_{e}(x) = 0$ and $h_{i}(x) \leq 0$. Three types of constraint functions are considered in this paper, i.e. functions describing state, the input, and the mixed state-input constraints denoted as $h_x(x)$, $h_u(u)$, and $h_{xu}(x,u)$, respectively. In practice, the state constraint is introduced to avoid violating joint position or velocity limits when controlling the robot. Input constraints are considered for describing the joint torque limit or the underactuation of floating robots. The mixed state-input constraint contains physically more intricate conditions such as mechanical power. Since the DTSCR is required to maintain stable contact while being controlled, we consider an additional constraint that is known as the contact wrench cone constrained to prevent slip and flip on contact surface. In particular, it is required that 
\begin{equation}
   h_{xc} (x,F_c) \leq 0, \quad h_{xc} (x,F_c) \coloneqq \mathbf{W}_c(x) F_c    
\end{equation}
where $\mathbf{W}_{c}(x)$ is a matrix describing the unilateral constraint using a polyhedral approximation of the friction cone of a surface \cite{caron2015stability}. The position at the contact point should be constant, which is identical to having null velocity on the robot's contact with respect to the surface contact. The zero velocity constraint corresponds to the state constraint. We will aim at controlling robots represented by the DTSCR model for desired output goals given the aforementioned constraints. 

 \begin{figure*}
\centering
\begin{minipage}[t]{\linewidth}
\includegraphics[width=\linewidth]{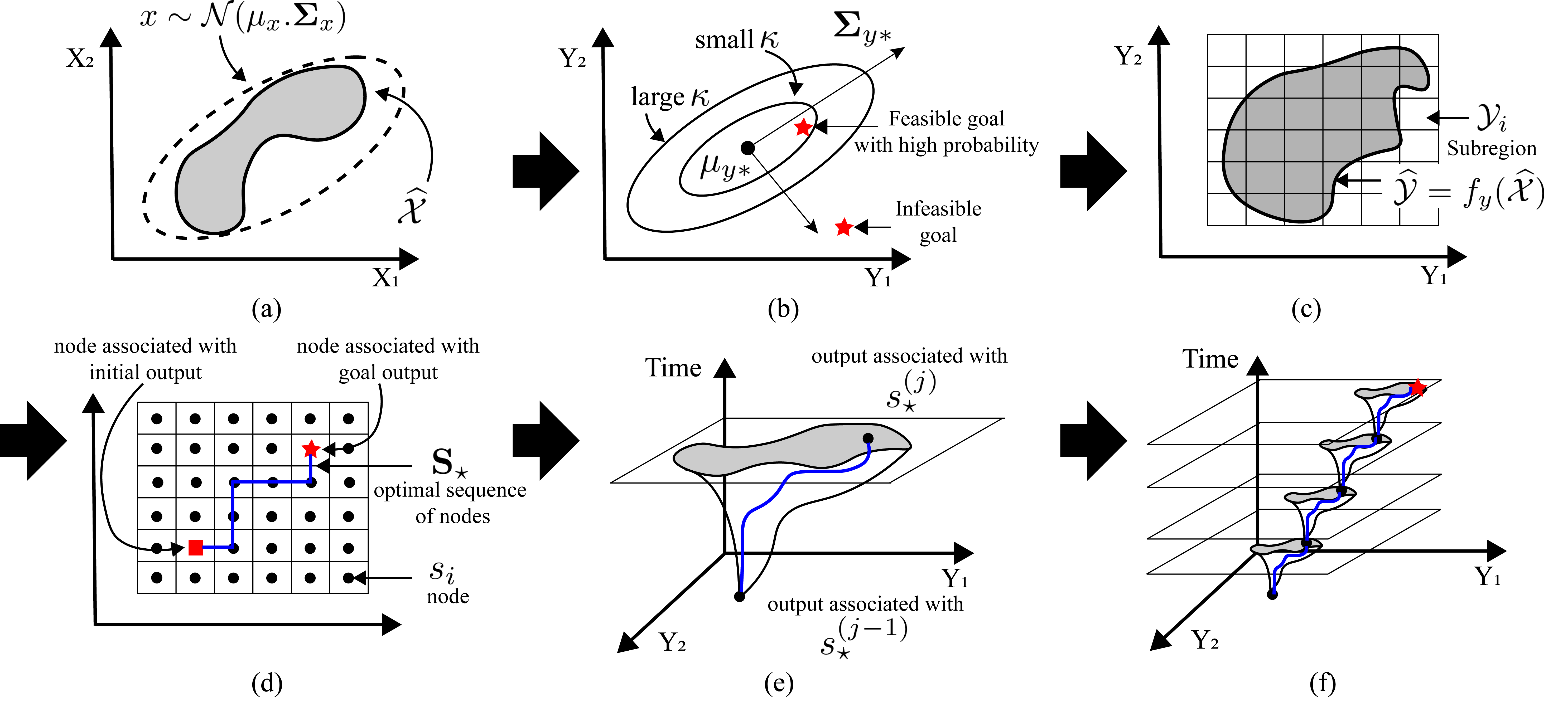}
\end{minipage}
\caption{Proposed method: (a) Generating random state vectors and obtaining a set of sample-wisely feasible states, Section \ref{sec3}.1, (b) Checking the feasibility goal output via output approximation, Section \ref{sec3}.2, (c) Defining output subregions and discrete node space, (d) Solving POMDP to obtain an optimal sequence of nodes, Section \ref{sec4}, (e) Obtaining reachable sets and optimizing trajectories between sequential nodes, Section \ref{sec5} and \ref{sec6} (f) Consecutively executing trajectory optimizations with reachable sets, Section \ref{sec5} and \ref{sec6}}
\label{Fig:sequence}
\end{figure*}

\subsection{Overall Scheme}
The proposed approach consists of three methods, which are sampling-based optimization, solving a POMDP, and reachability-based optimal control. The overall procedure is depicted in Fig. \ref{Fig:sequence}. The first part of the approach aims to obtain feasible states using sampling with respect to the given constraints as shown in Fig. \ref{Fig:sequence}(a), then to approximately check whether the goal output is feasible via output approximation as shown in Fig. \ref{Fig:sequence}(b). If the goal is achievable, we formulate and solve a POMDP process to create multiple tractable sub problems as shown in Fig. \ref{Fig:sequence}(c) and (d). Based on an optimal sequence of nodes, we obtain reachable sets and perform trajectory optimization between neighboring subregions associated with two sequential nodes, then we recursively iterate this process for all sequence of nodes until reaching the final goal output as shown in Fig \ref{Fig:sequence}(e) and (f). By connecting all trajectories, we obtain the entire trajectory in an efficient way. 

\section{Sample-Based Optimization}
\label{sec3}
In this section, we obtain a sequence of subregions in output space considering initial and goal states. Our method breaks down the end-to-end trajectory generation problem into multiple intermediate points sequentially connected. We start by creating random samples from a Gaussian distribution $x\sim\mathcal{N}(\mu_{x}, \mathbf{\Sigma}_x)$ where $\mu_x \in \mathbb{R}^{n_x}$ and $\mathbf{\Sigma}_x \in \mathbb{S}_{>0}^{n_x}$ are the mean and the covariance matrix, respectively, that is, $\mu_x \coloneqq \mathbb{E}[x]$ and $\mathbf{\Sigma}_x\coloneqq \mathbb{E}[(x-\mu_x)(x-\mu_x)^{\top}]$.

\subsection{Update for Random State Samples}
We consider a set of states fulfilling desired constraints. Since the Monte-Carlo method is very inefficient, we apply a least-square QP process to obtain feasible states. Let us consider $n_{e}^{c}$ state equality constraints and $n_{i}^{c}$ state inequality constraints. We define constraint functions as $h_{x,e[k_e]}(x)$ and $h_{x,i[k_i]}(x)$ where $k_e \in [1,n_{e}^c]_\mathbb{N}$ and $k_i \in [1,n_{i}^c]_\mathbb{N}$ are indices. Let us define Jacobian matrices and error vectors as follows:
\begin{subequations}
\begin{align}
    &\mathbf{J}_{e}(x) \coloneqq \mathrm{Vertcat}(J_{e[k]}(x):\forall k \in [1,n_{e}^c]_{\mathbb{N}}) \\
    &\mathbf{J}_{i}(x) \coloneqq \mathrm{Vertcat}(J_{i[k]}(x):\forall k \in [1,n_{i}^c]_{\mathbb{N}}) \\
    &\mathbf{e}_{e}(x) \coloneqq -\mathrm{Vertcat}(h_{x,e[k]}(x): \forall k \in [1,n_{e}^{c}]_{\mathbb{N}})\\
    &\mathbf{e}_{i}(x) \coloneqq  v_{\backslash i}^{int} - \mathrm{Vertcat}(h_{x,i[k]}(x): \forall k \in [1,n_{i}^{c}]_\mathbb{N})
\end{align}
\end{subequations}
where $J_{e[k]}(x) \coloneqq \frac{\partial h_{x,e[k]}}{\partial x}(x)$ and $J_{i[k]}(x)\coloneqq \frac{\partial h_{x,i[k]}}{\partial x}(x)$ denote the Jacoabians for equality and inequality constraint functions, respectively. $v_{\backslash i}^{int}$ is an arbitrary interior vector satisfying inequality constraints used as an attractor. The main idea for obtaining the state fulfilling constraints is to iteratively update the sampled state using the state increment $\Delta x$ until the constraints are satisfied. The state increment $\Delta x$ is obtained by using the QP method as follows:
\begin{equation} \label{opt_state}
\begin{split}
    \min_{\Delta x,w}& \quad \| w \|_{2}^{2} \\
    \textrm{s.t}& \quad \mathbf{J}_{e}(x)\Delta x \leq \mathbf{e}_e(x) + w \\
    &\quad \mathbf{J}_{i}(x)\Delta x \leq \mathbf{e}_{i}(x) 
\end{split}
\end{equation}
and we update the state $x$ using the optimal variable $\Delta x$. For numerical efficiency, we discard state samples if QP does not converge. Let us define a set $\mathcal{X}$ consisting of states fulfilling the constraints as follows:
\begin{equation} \label{state_stisfying}
\begin{split}
\mathcal{X} \coloneqq \{x \in \mathbb{R}^{n_x} :& \| h_{x,e[k_e]} (x) \| \leq \varepsilon, \: \forall k_{e} \in [1, n_{e}^{c}]_{\mathbb{N}}\: \\
& h_{x,i[k_i]}(x) \leq 0, \: \forall k_{i} \in [1, n_{i}^{c}]_{\mathbb{N}} \} 
\end{split}
\end{equation}
where $\varepsilon$ is a desired tolerance. 

For the next step, we take all elements of $\mathcal{X}$ and check whether they fulfill the input, mixed state-input, and contact force constraints. To achieve this, we formulate an optimization problem with a quadratic cost function as follows:
\begin{equation} \label{opt1}
\begin{split}
    \min _{F_c} \quad &F_{c}^{\top} \mathbf{Q}_{c} F_{c} + u^{\top} \mathbf{Q}_{u} u \\
    \mathrm{s.t.} \quad & x(t_{k+1}) = f_{\mathfrak{D}}(x(t_{k}),u,F_{c})\\
    & h_{u[k_u]}(u) \leq 0, \: \forall k_u \in [1, n_u^c]_\mathbb{N} \\
    & h_{xu[k_{xu}]}(x(t_k), u) \leq 0, \:  \forall k_{xu} \in [1, n_{xu}^c]_\mathbb{N}  \\
    & h_{xc}(x(t_k),F_{c})  \leq 0,\: x(t_{k}), x(t_{k+1}) \in \mathcal{X}
\end{split}
\end{equation}
where $\mathbf{Q}_{c} \in \mathbb{S}^{n_c}_{>0}$ and $\mathbf{Q}_u \in \mathbb{S}^{n_u}_{>0}$ are weighting matrices for the cost. By solving the optimization problem (\ref{opt1}) for all $x(t_{k}) \in \mathcal{X}$, we obtain a set of states, $\hat{\mathcal{X}}$, that fulfills all desired constraints.

\subsection{Output Space Approximation}
In this subsection, we check the feasibility of reaching the desired goal output. To do so, we approximate the output samples with a Gaussian distribution $y^{*} \sim \mathcal{N}(\mu_{y^{*}}, \mathbf{\Sigma}_{y^{*}})$ \cite{hendeby2007nonlinear}. The mean and covariance matrix obtained after neglecting higher order terms are 
\begin{subequations} \label{output_compute}
\begin{align}
    \mu_{y^{*}} &\coloneqq f_y(\mu_x) + \mathrm{Vec} \left[\mathrm{tr}(\mathbf{H}_{y,i}(\mu_x)\mathbf{\Sigma}_x)\right]_{i=1}^{n_y}\\
    \mathbf{\Sigma}_{y^{*}} &\coloneqq \mathbf{J}_{y}(\mu_x) \mathbf{\Sigma}_{x} \mathbf{J}_{y}^{\top}(\mu_x)  \nonumber\\
    &  + \frac{1}{2}\mathrm{Mat}\left[\mathrm{tr}(\mathbf{\Sigma}_{x}\mathbf{H}_{y,i}(\mu_x) \mathbf{\Sigma}_{x}\mathbf{H}_{y,j}(\mu_x))\right]_{i,j=1}^{n_y,n_y}
\end{align}
\end{subequations}
where $\mathbf{J}_{y}(\mu)$ and $\mathbf{H}_{y,i}(\mu)$ denote the Jacobian matrix of the output function $f_y(\mu)$ and the $2$nd derivative matrix of the output function $f_{y,i}(\mu)$, for the $i$-th element. In particular, 
\begin{equation}
\begin{split}
    \mathbf{J}_{y}( \mu) \coloneqq \frac{\partial f_y}{\partial x} (\mu),
    \mathbf{H}_{y,i}(\mu) \coloneqq \left[ \begin{array}{ccc}
         \frac{\partial^{2} f_{y,i} (\mu)}{\partial x_{1}^2}& \dots & \frac{\partial^{2} f_{y,i} (\mu)}{\partial x_{1} \partial x_{n_x}}  \\
         \vdots & \ddots & \vdots \\
         \frac{\partial^{2} f_{y,i} (\mu)}{\partial x_{n_x} \partial x_{1}} & \dots & \frac{\partial^{2} f_{y,i} (\mu)}{ \partial x_{n_x}^2}
    \end{array} \right]
\end{split}
\end{equation}
where $\mathbf{J}_{u}(\mu) \in \mathbb{R}^{n_y \times n_x}$ and $\mathbf{H}_{y,i}(\mu) \in \mathbb{R}^{n_x \times n_x}$ is a symmetric matrix. We construct a probabilistic ellipsoid in the output space to approximate whether an output sample $y^{*}$ is feasible. We define a set of outputs that lie inside an ellipsoid $\mathcal{E}_{\kappa}$ with
\begin{equation}\label{ellip}
    \mathcal{E}_{\kappa} \coloneqq \{ y \in \mathbb{R}^{n_y}: (y - \mu_{y^{*}})^{\top} \mathbf{\Sigma}_{y^{*}}^{-1} (y - \mu_{y^{*}}) \leq \kappa \}
\end{equation}
where $\kappa$ is a coefficient determined by the cumulative probability of the Chi-square distribution. For instance, $\kappa=5.991$ for $\mathrm{Pr}(y^{*}  \in \mathcal{E}_{\kappa}) = 0.95$ and $y^{*}\in \mathbb{R}^{2}$. Our method to check if a goal output $y_g$ is interior to $\mathcal{E}_\kappa$ is more efficient than using a Monte Carlo method, because we only need to compute $\mu_{y^{*}}$ and $\mathbf{\Sigma}_{y^{*}}$ using the mean and covariance matrix of the samples using (\ref{output_compute}). 

\section{POMDP for a Sequence of Subregions}
\label{sec4}
After checking that the desired output goal $y_{g}$ is located at the interior of the ellipsoid $\mathcal{E}_{\kappa}$ in ($\ref{ellip}$), the end-to-end trajectory generation problem is divided using intermediate points in the output space. To start the process, we define output subregions:
\begin{equation}
    \mathcal{Y}_{i} \coloneqq \left\{ y \in \mathbb{R}^{n_y}:  \| y - y_{c,i}  \|_{\infty}< \varepsilon_{y} \right\}
\end{equation}
where $y_{c,i} \in \mathbb{R}^{n_y}$ denotes the center of the output subregion $\mathcal{Y}_{i}$ and $i\in[1,m]_{\mathbb{N}}$ where $\bigcup_{i\in [1,m]_\mathbb{N}} \mathcal{Y}_{i} = \overline{\mathcal{Y}} \subset \mathbb{R}^{n_y}$. Also, we obtain a set of outputs $\hat{\mathcal{Y}} \coloneqq \hat{f}_{y}(\hat{\mathcal{X}}) = \{ y \in \mathbb{R}^{n_y}: y = f_y(x) , x \in \hat{\mathcal{X}} \}$ where $\hat{f}_y : \hat{\mathcal{X}} \	\rightrightarrows \hat{\mathcal{Y}}$. To formulate our problem as a POMDP, we define discrete nodes associated with the previous subregions as follows:
\begin{equation}
    s_{i} = \mathrm{node}( \mathcal{Y}_{i}) , \quad i \in [1, m]_{\mathbb{N}}
\end{equation}
where $\mathbf{S} \coloneqq \{ s_1, \dots, s_m \}$. Based on these nodes, we transform the problem to a POMDP. We will formulate the probability of observations using the sampled states.

\begin{definition}
(POMDP) Partially Observable Decision Making Process is defined as a tuple $\mathbf{P} = (\mathbf{S}, \mathbf{A}, \mathbf{O}, \mathbf{T}, \mathbf{Z} )$:
\begin{itemize}
    \item $\mathbf{S}$ is a finite set of nodes, $\mathbf{S} \coloneqq \{ s_1, \cdots, s_{m_s}\}$ 
    \item $\mathbf{A}$ is a finite set of actions, $\mathbf{A} \coloneqq \{a_1, \cdots, a_{m_a} \}$ 
    \item $\mathbf{O}$ is a finite set of observations, $\mathbf{O} \coloneqq \{ o_{1}, \cdots, o_{m_o}\}$
    \item $\mathbf{T}$ is the transition dynamics $\mathbf{T}(s',s,a)$ defining the transition from $s\in \mathbf{S}$ to $s'\in \mathbf{S}$ after taking an action $a \in \mathbf{A}$.
    \item $\mathbf{Z}$ is the observation $\mathbf{Z} (s,a,o)$ consisting of the probability of observing $o\in \mathbf{O}$ after taking an action $a\in \mathbf{A}$ from node $s\in S$.
\end{itemize}
\end{definition}

The problem concerning this section is on finding a sequence of feasible subregions towards an output goal using POMDP tools and analysis. 

\begin{definition}
(Markov Policy) A Markov policy $\mathbf{\Pi}$ is defined as a sequence: $\mathbf{\Pi} \coloneqq \{ a^{(1)}, \cdots,  a^{(n)}\}$. $a^{(j)} \in \mathbf{A}$, where $a^{(j)}: \mathbf{S} \rightarrow \mathbf{S}$ is a measurable map from a node to another one, $j\in [1,n]_{\mathbb{N}}$ .
\end{definition}

We convert the POMDP into a belief MDP. Belief $b[s_i]$ is defined with respect to discrete nodes $s_i \in \mathbf{S}$. Let suppose $b=b[s^{(j)}]$, $b' = b[s^{(j+1)}]$, and $a= a^{(j)}$ where $s^{(j)}$ represents the node for the $j$-th step of the POMDP. The belief transition function, $\Gamma(b,a,b')$, is equal to
\begin{subequations}
\begin{align}
    &\Gamma(b,a,b') = \sum_{o\in \mathbf{O}} \mathrm{Pr}(b'|b,a,o)\mathrm{Pr}(o|b,a) \\
    &\mathrm{Pr}(b'|b,a,o) = \left\{ \begin{array}{ll} 1, & \textrm{if belief update returns $b'$} \\ 0, & \textrm{otherwise} \end{array} \right. \\
    &\mathrm{Pr}(o|b,a) = \sum_{s'\in \mathbf{S}}\mathbf{Z}(s',a,o)\sum_{s \in \mathbf{S}} \mathbf{T}(s',s,a)b \textrm{.}
\end{align}
\end{subequations}
The key challenges of this POMDP are on defining useful observations and on finding their conditional probability. Let us consider that $\mathcal{Y} = \mathcal{Y}^{(j)}$ and $\mathcal{Y}' = \mathcal{Y}^{(j+1)}$ associated with the nodes $s^{(j)}$ and $s^{(j+1)}$. We propose to define observations as the set of feasible states after taking an action $a$, i.e.
\begin{equation}
\begin{split}
    \hat{O} \coloneqq &\{ z \in \mathbb{R}^{n_x}:  z  = x(t_k), \: \exists (F_{c}, u) \textrm{ in } (\ref{opt1}) \textrm{ with }\\
    & f_y(x(t_k)) \in \mathcal{Y} \cap \hat{\mathcal{Y}}, \: f_y(x(t_{k+1}) \in \mathcal{Y}' \cap \hat{\mathcal{Y}} \} \textrm{.}
\end{split}
\end{equation}
where $\mathcal{Y}$ is the subregion before taking the action $a$. If $z_1 \in \hat{O}$, it holds that there exists at least one sample connecting $f_y(z_1)$ to another output in the subregion $\mathcal{Y}'$ satisfying the constraints. Considering the above observations, we define the conditional probability as 
\begin{equation}
    \mathbf{Z}(s',a,o) \coloneqq \mathrm{Pr}(o|s',a)  = \mathrm{card}( \hat{O} )/\mathrm{card}( \mathcal{Y}' \cap \hat{\mathcal{Y}}).
\end{equation}
Let us focus on the reward and transition dynamics. As a heuristic, a higher number of feasible samples falling into a subregion implies a higher probability of reaching it. Therefore we define the reward
\begin{equation}
    \mathfrak{R}(s_i,a) \coloneqq K_{r} \mathrm{card}( \mathcal{Y}_{i} \cap \hat{\mathcal{Y}} )/\mathrm{card}(\hat{\mathcal{Y}}) + \eta_{i}
\end{equation}
where $\mathcal{Y}_i$ is the subregion associated with node $s_i \in \mathbf{S}$. $K_r \in \mathbb{R}_{>0}$ and $\eta_{i} \in \mathbb{R}$ are the gain and reward offset, respectively. We take $\eta_i$ to be large when $y_{g} \in \mathcal{Y}_{i}$. Also, to avoid unsafe output regions we set $\eta_{i}$ to large negative values.
% In order to find the transition dynamics from the set of samples, we employ the Principle Singular Vector (PSV) of the set of samples. 

\begin{definition}
Consider a node $s_{i} \in \mathbf{S}$ associated with an output subregion $\mathcal{Y}_{i}$ and a set  $\hat{\mathcal{Y}} = f_y (\hat{\mathcal{X}})$. Consider $\mathbf{\Sigma}_{\mathcal{Y}_{i}}$ being the covariance matrix for the set $\mathcal{Y}_{i} \cap \hat{\mathcal{Y}}$, that is, $\mathbf{\Sigma}_{\mathcal{Y}_{i}}=\mathbb{E}[(y - \mu_{y})(y - \mu_{y})^{\top}]$ and $\mu_{y} = \mathbb{E}[y]$ where $y \in \mathcal{Y}_{i} \cap \hat{\mathcal{Y}}$. A principle singular vector is defined as
\begin{equation}
    \mathcal{V}(s_{i}) = \mathrm{col}(\mathbf{V}_{\mathcal{Y}_{i}})_{k}, \quad \sigma_{k}(\mathbf{\Sigma}_{\mathcal{Y}_i}) = \overline{\sigma} (\mathbf{\Sigma}_{\mathcal{Y}_i}) \textrm{.}
\end{equation}
where $\mathbf{\Sigma}_{\mathcal{Y}_i} = \mathbf{V}_{\mathcal{Y}_{i}}^{\top} \mathbf{\Lambda}_{\mathcal{Y}_{i}} \mathbf{V}_{\mathcal{Y}_{i}}$, $\mathbf{\Lambda}_{\mathcal{Y}_{i}} = \mathrm{diag}(\sigma_{1}, \dots,\sigma_{n_y})$, and $\sigma_{k}$ denotes the singular value of $\mathbf{\Sigma}_{\mathcal{Y}_{i}}$.
\end{definition}

For defining the transition dynamics, let an action $a \in \mathbf{A}$ map state $s \in \mathbf{S}$ to $s' \in \mathbf{S}$. $d$ is a vector in our grid world defined as $d \coloneqq (y_{c}' - y_{c})$ where $y_c$ and $y_c'$ are the centers of subregions associated with $s$ and $s'$, respectively. We define the transition dynamics as
\begin{subequations}\label{transition_dyn}
\begin{align}
    \mathbf{T}(s',s,a) &\coloneqq  \left\{ \begin{array}{ll} \mathcal{T}(s', s, a) / \varpi,& \quad  \textrm{if }\varpi \neq 0, \\
     0, & \quad \textrm{else} \end{array}\right. \\ 
    \mathcal{T}(s',s, a) &\coloneqq \mathrm{max}\left\{ 0, \mathcal{V}^{\top}(s) \mathbf{a} \right\}
\end{align}    
\end{subequations}
where $\varpi = \sum_{a' \in \mathbf{A}} \mathcal{T}(s', s, a')$ denotes a normalization constant.

\begin{proposition}
Let $s=s_{i} \in \mathbf{S}$ and the corresponding subregion be $\mathcal{Y}_{i}$. If the output samples are uniformly distributed in $\mathcal{Y}_{i}$, then $\mathbf{T}(s', s, a) =0$. 
\end{proposition}
\begin{proof}
Since the samples are uniformly distributed, it is possible to select any unit vector in $\mathbb{R}^{n_y}$ as the PSV of $s_i$, that is, $\mathbf{R}_a\mathcal{V}(s_i)$ where $\mathbf{R}_a \in \mathrm{SO}(3)$ is a rotation matrix. If we select the rotation matrix $\mathbf{R}_a$ such that $d^{\perp} = \mathbf{R}_a \mathcal{V}(s_i)$, which is orthogonal to $\mathcal{V}(s_i)$, it follows that $\mathcal{T}(s',s,a) = \mathcal{V}^{\top}(s_{i}) d = (d^{\perp})^{\top} d = 0$ for all $a \in \mathbf{A}$.
\end{proof}

We now solve a finite-horizon belief MDP. The optimal policy, denoted by $\mathbf{\Pi}_{\star}$, is obtained by solving the Bellman equation as follows:
\begin{equation}
    \mathfrak{D}_{\star}(b) = \max_{a \in \mathbf{A}} [r(b,a) + \gamma \sum_{o \in \mathbf{O}}\textrm{Pr}(o|b,a) \mathfrak{D}_{\star}(\Gamma(b,a,o))] 
\end{equation}
where $r(b,a) = \sum_{s \in \mathbf{S}} b(s)\mathfrak{R}(s,a)$ denotes the belief reward. The result of the DP provides an optimal Markov policy which we transform to a sequence of nodes as
\begin{subequations}
\begin{align}
  \mathbf{\Pi}_{\star} &= \{a_{\star}^{(1)}, \cdots, a_{\star} ^{(n_{\pi})} \} \\
  \mathbf{S}_{\star} &= \{ s_{\star}^{(1)}, \cdots, s_{\star}^{(n_{\pi})} \}  
\end{align}
\end{subequations}
where $a_{\star}^{(i)}$ and $s_{\star}^{(i)}$ are $i$-th action and node in a sequence toward reaching the final output goal, respectively. The sequence of subregions in output space is
\begin{equation} \label{pair}
\begin{split}
    \mathbf{Y}_{\star} &= \{\mathcal{Y}_{\star}^{(1)}, \cdots, \mathcal{Y}_{\star}^{(n_{\pi})} \} \textrm{.}
\end{split}
\end{equation}
Based on the generated sequence of subregions, we will generate trajectories using reachability analysis connecting subregions in $\mathbf{Y}_{\star}$ in the next section. 

\section{Reachability Analysis}
\label{sec5}
In this section, we define discrete reachable sets and propose the way to obtain the reachable sets via optimizations. To overcome the computational complexity of propagation algorithm for the reachable sets, a method propagating the boundary samples is proposed and analyzed in the views of computational complexity.   

\subsection{Optimization-based Reachability Analysis}
We define reachable sets in continuous time domain as
\begin{equation} \label{continuous_set}
\begin{split}
    \mathcal{R}_{x}^{C}&(t,x_0) \coloneqq \{ x(t) : \exists u([t_0, t]), \exists F_{c}([t_0, t])\\
    & h_{x[k_x]}(x(\tau)) \leq 0 ,  \: \forall k_x \in [1,n_x^c]_{\mathbb{N}}\\
    & h_{xu[k_{xu}]}(x(\tau),u(\tau)) \leq 0,  \: \forall k_{xu} \in [1, n_{xu}^c]_{\mathbb{N}} \\
    & h_{xc}(x(\tau),F_c(\tau)) \leq 0, \: u(\tau) \in \mathcal{U} \\
    & \dot{x}(\tau) = f_{\mathcal{C}}(x(\tau), u(\tau), F_c(\tau)) \\
    &  x(t_0) = x_0 \in \mathcal{R}_x^{C}(t_0,x_0), \:  \tau \in [t_0, t]\}
\end{split}
\end{equation}
$\mathcal{R}_{x}^{C}(t_0, x_0)$ is equal to $\{ x_0 \}$ which means that the initial state fulfills all constraints.
\begin{definition}
Let $x_0 \in \mathcal{R}_x^{C}(t_0,x_0)$, be an initial state and $t \in [t_0, t_f]_{d}^{\Delta t}$ be an arbitrary time interval. We define a reachable set in discrete time domain as:
\begin{equation}\label{def:reachable_set}
\begin{split}
	\mathcal{R}_{x}^{D}& \left(t,x_{0} \right) \coloneqq \{ x(t) : \exists u([t_0, t]_{d}^{\Delta t}), \exists F_{c}([t_{0}, t]_{d}^{\Delta t}) \\
	& h_{x[k_x]}(x(\tau)) \leq 0 , \: \forall k_x \in [1,n_x^c]_{\mathbb{N}}\\
    & h_{xu[k_{xu}]}(x(\tau),u(\tau)) \leq 0, \: \forall k_{xu} \in [1, n_{xu}^c]_{\mathbb{N}} \\
    & h_{xc} (x(\tau),F_c(\tau))  \leq 0, \: u(\tau) \in \mathcal{U} \\
    & x(\tau + \Delta t) = f_{\mathfrak{D}}(x(\tau),u(\tau),F_c(\tau)) \\
    & x(t_{0}) = x_0 \in \mathcal{R}_{x}^{D}(t_0,x_0), \: \tau \in [ t_{0}, t]_{d}^{\Delta t} \}
\end{split}
\end{equation}
where $\Delta t > 0$ is the discretization step or sampling period for our discrete model. 
\end{definition}
We extend the reachable set defined above for the finite discrete time interval $[t_0, t_f]_{d}^{\Delta t}$ as $ \mathcal{R}_{x}^{D}([t_0,t_f]_{d}^{\Delta t},x_0) \coloneqq \bigcup_{t \in [t_0, t_f]_{d,\Delta t}} \mathcal{R}_{x}^{D}(t,x_0) $. For any $t_f< \infty$, the reachable set satisfies the following bound $\mathcal{R}_{x}^{D}([t_0,t_f]_{d}^{\Delta t}$, $ x_0) \subseteq \mathcal{R}_{x}^{D}([t_0,+\infty),x_0) \subseteq \mathcal{R}_{x}^{C}([t_0,+\infty),x_0) \subsetneq \mathcal{X}$.

Consider $x_0$ and $\mathcal{R}_{x}^{D}([t_0, t_{k}]_{d}^{\Delta t}, x_0)$. A random input is drawn from a Gaussian distribution $u\sim\mathcal{N}(\mu_u,\mathbf{\Sigma}_{u})$ at each instant of time with the input set $\mathcal{U}$ defined as the collection of inputs fulfilling input constraint. We define a QP to check for feasible contact forces, i.e.
\begin{equation}\label{opt2}
\begin{split}
    \min _{F_c, \Delta x} \quad &F_{c}^{\top} \mathbf{Q}_{c} F_{c} + \Delta x^{\top} \mathbf{Q}_{x} \Delta x \\ 
    \mathrm{s.t.}\quad& x(t_{k+1}) = f_{\mathfrak{D}}(x(t_{k}),u(t_k),F_{c}) \\
    & x(t_{k}) \in \mathcal{R}_{x}^{D}([t_0,t_{k}]_{d}^{\Delta t},x_0)\\
    & h_{x[k_x]}(x(t_{k+1})) \leq 0, \: \forall k_{x} \in [1, n_x^c]_{\mathbb{N}} \\
    & h_{xu[k_{xu}]}(x(t_{k}),u(t_k) \leq 0, \:\forall k_{xu} \in [1, n_{xu}^c]_\mathbb{N} \\
    & h_{xc} (x(t_k),F_c) \leq 0, \: u(t_{k}) \in \mathcal{U}\\
    & \Delta x = x(t_{k+1}) - x(t_{k})  
\end{split}
\end{equation}
The reachable set at the next time instance $t_{k+1}$ is obtained by collecting the $x(t_{k+1})$ updated by the QP solution $\Delta x$ in (\ref{opt2}). Then, we can extend the reachable set over $[t_0, t_{k+1}]^{\Delta t}_{d}$ such that $\mathcal{R}_{x}^{D}([t_0, t_{k+1}]_d^{\Delta t},x_0) \coloneqq \mathcal{R}_{x}^{D}([t_0,t_k]_d^{\Delta t},x_0) \cup \mathcal{R}_{x}^{D}(t_{k+1},x_0)$. Also, we define the set of outputs associated with reachable states such as $\mathcal{R}_{y}^{D}(t,x_0) = \hat{f}_y(\mathcal{R}_{x}^{D}(t,x_0))$ and $\mathcal{R}_{y}^{D}([t_0,t_f]_d^{\Delta t},x_0) = \hat{f}_y(\mathcal{R}_{x}^{D}([t_0,t_f]_{d}^{\Delta t},x_0))$.

\begin{theorem} \label{thm1}
Suppose that $\mathcal{R}_{x}^{D} (t_0,x_0)$ is a non-empty set and the desired output, $y_d$, is given. Let us assume that the set, $\mathcal{R}_{y}^{D}([t_0,t_f]_d^{\Delta t}, x_0)$, is compact, connected, and
\begin{equation}
    y_{d} \in \mathcal{R}_{y}^{D}([t_0,t_f]_d^{\Delta t}, x_0) \textrm{.}
\end{equation}
At least, one state trajectory $\mathbf{\Psi} \coloneqq \{ \xi(t_0), \xi(t_1), \dots, \xi(\tau) \}$ exists such that $f_y(\xi(\tau)) = y_d$ where $\tau \leq t_f$.
\end{theorem}
\begin{proof}
Since $\mathcal{R}_{y}^{D}([t_0,t_f]_d^{\Delta t}, x_0)$ is compact and $f_y$ is continuous, $\mathcal{R}_{x}^{D}([t_0,t_f]_d^{\Delta t}, x_0)$ is closed and $f_y^{-1}$ is also continuous. Then, $\mathcal{R}_{x}^{D}([t_0,t_f]_d^{\Delta t}, x_0)$ is connected because $\mathcal{R}_{y}^{D}([t_0,t_f]_d^{\Delta t}, x_0)$ is connected and $f_y^{-1}$ is continuous. Therefore, there exists at least one trajectory connecting $x_0$ to $x_f$ satisfying $f_y(x_f) = y_d$ in $\mathcal{R}_{x}^{D}([t_0,t_f]_d^{\Delta t}, x_0)$.   
\end{proof}

\begin{corollary} \label{col1}
The set, $\mathcal{R}_{y}^{D}([t_0,t_k]_d^{\Delta t},x_0)$, is compact.
\end{corollary}
\begin{proof}
Let us consider $\mathrm{ghull}(\mathcal{R}_{y}^{D}([t_0,t_k]_d^{\Delta t},x_0))$, being compact. By the Heine$-$Borel theorem, all closed subsets of a compact set are also compact. Since $\mathcal{R}_{y}^{D}([t_0,t_k]_d^{\Delta t},x_0) \subset \mathrm{ghull}(\mathcal{R}_{y}^{D}([t_0,t_k]_d^{\Delta t},x_0))$, the reachable set $\mathcal{R}_{y}^{D}([t_0,t_k]_d^{\Delta t},x_0)$ is compact.  
\end{proof}
\begin{corollary} \label{col2}
Suppose that $x_0 \in \mathcal{R}_{x}^{D}(t_0, x_0)$ and $f_y$ is continuous. Then, a set, $\mathcal{R}_{y}^{D}([t_0,t_k]_d^{\Delta t},x_0)$, is connected. 
\end{corollary}
\begin{proof}
Consider three sets: $\mathcal{F}_{1} = \mathcal{R}_{x}^{D}([t_{0},t_{k-1}]_d^{\Delta t}, x_0)$, $\mathcal{F}_{2} = \mathcal{R}_{x}^{D}(t_k, x_0)$, and $\mathcal{F}_{3} = \mathcal{F}_2 \cup \mathcal{F}_{1}'$, where $\mathcal{F}_{1}'$ is the collection of states $x\in \mathcal{F}_1$ producing the next feasible state via the optimization (27) with respect to $x \in \mathcal{F}_{1}$. Then, $\mathcal{R}_{x}^{D}([t_0, t_{k}]_d^{\Delta t}, x_0)=\mathcal{F}_1\cup\mathcal{F}_2 = \mathcal{F}_{2} \cup \mathcal{F}_{3}$. Let us consider arbitrary two sets $\mathcal{H}_1$ and $\mathcal{H}_2$ satisfying $\mathcal{R}_{x}^{D}([t_0, t_{k}]_d^{\Delta t}, x_0) = \mathcal{H}_1 \cup \mathcal{H}_2$ with $\mathcal{H}_1 \cap \mathcal{H}_2 = \emptyset$. Let $x_h \in F_{1}'$ and suppose $x_h \in \mathcal{H}_{1}$. Then, $\mathcal{H}_1 \cap \mathcal{F}_{1} \neq \emptyset$ and $\mathcal{H}_1 \cap \mathcal{F}_{3} \neq \emptyset$. This implies that $\mathcal{F}_{1} \subseteq \mathcal{H}_1$ and $\mathcal{F}_{3} \subseteq \mathcal{H}_1$, hence, $\mathcal{H}_{2} = \emptyset$. This proves that $\mathcal{R}_{x}^{D}([t_0,t_k]_d^{\Delta t}, x_0)$ is connected. Since the mapping $f_y$ is continuous, we also conclude that the set $\mathcal{R}_{y}^{D}([t_0, t_k]_d^{\Delta t}, x_0)$ is connected. 
\end{proof}

\subsection{Propagation of Boundary States}
The basic algorithm for reachability analysis suffers from exponential complexity with respect to the number of time steps. Although the previous POMDP contributes to reducing the time horizon to be checked for reachability analysis, full-state propagation would still result in heavy computational burden. In this section, we propose a method for reducing the computational complexity of the algorithm by only propagating selected states. This approach results in more conservative reachable sets. We implement the propagation of boundary states by solving (25) for only state samples $x(t_{k}) \in \mathcal{B}_{x}^{D}([t_0, t_k]_d^{\Delta t},x_0)$ such that 
\begin{equation}
\begin{split}
    \mathcal{B}_{x}^{D}([t_0,t_k]_d^{\Delta t},x_0) \coloneqq \{& x \in \mathbb{R}^{n_x} : x\in \mathcal{R}_{x}^{D}([t_0,t_k]_d^{\Delta t},x_0),\\
    & f_y(x) \in \mathrm{gbd}(\mathcal{R}_{y}^{D}([t_0,t_k]_d^{\Delta t} ,x_0)) \} \textrm{.}
\end{split}
\end{equation}
We then compute the reachable set by collecting the updated states and extend the time horizon. In this way, the complexity becomes linear with respect to the number of boundary samples, $\mathrm{card}(\overline{\mathcal{B}}^{D}([t_0,t_k]_d^{\Delta t},x_0))$. In order to replace full-state propagation with boundary-state propagation, we show that the set of reachable outputs is compact and connected. First, the set $\overline{\mathcal{R}}_{y}^{D}([t_0,t_k]_d^{\Delta t},x_0)$ is compact, because we are able to obtain the hull of the set as shown in Corollary \ref{col1}. Next, we prove the reachable set $\mathcal{R}_{y}^{D}$ is connected.
\begin{corollary}
    Suppose that $x_0 \in \mathcal{R}_x^{D}(t_0,x_0) \neq \emptyset$ and $f_y$ is continuous. Then, $\overline{\mathcal{R}}_{y}^{D}([t_0,t_k]_d^{\Delta t},x_0)$ is connected.
\end{corollary}
\begin{proof}
The proof is similar to that of Corollary 2 and therefore is omitted.
\end{proof}

\subsection{Computational Complexity Analysis}
\label{sec_complexity}
We analyze the computational complexity to compare the efficiency of the propagation of full states and that of boundary states. There exists many algorithms to obtain the concave hull from the set of data \cite{galton2006region,duckham2008efficient,moreira2007concave}. They have $\mathcal{O}(n^{3})$ or $\mathcal{O}(n\log n)$ time complexity with $n$ data in $2$-D space. General QPs are non-deterministic polynomial-time hard, which means the algorithms are more complex than the polynomial time complexity to be solved. In the case that the QP is convex, it is widely known that the time complexity of the QP is $\mathcal{O}(m^{3})$ where $m$ is the number of decision variables.

Based on the aforementioned discussion, we can compare the computational complexity of two cases: propagation of full states and propagation of boundary states. Let us consider $N_t$ steps over the time interval $[t_0,t_k]_d^{\Delta t}$ where $\Delta t = (t_k- t_0)/N_t$, and $N_u$ is the number of input samples. For each propagation method, the computational complexity can be represented as $\mathbf{C}_{f} \sim  \mathcal{O} (\sum_{i=1}^{N_t} N_u^{i} (n_c + n_x)^{3} )\approx \mathcal{O} \left( N_u^{N_t+1} (n_c+n_x)^{3} \right)$, and $\mathbf{C}_{b} \sim  \mathcal{O} (\sum_{i=1}^{N_t} N_b(n_c+n_x)^{3} + (iN_b)^{3}) \approx \mathcal{O}( N_tN_b(n_c+n_x)^3 + N_t^4N_b^3) $ 
where $\mathbf{C}_{f}$, $\mathbf{C}_b$, and $N_b$ denote the complexity of full state propagation, that of boundary state propagation, and the number of boundary samples. Normally, a set of boundary samples contains much smaller samples than a set of entire states, that is, $N_b \ll N_u$. The effect of the boundary sampling on computational complexity becomes significantly advantageous in terms of the number of time steps. We will show the comparison of the computational complexity using an example in the simulation section.   

\section{Optimal Control}
\label{sec6}
In this work, we describe the use of sequential optimal control for nonlinear programs without constraint softening. Instead of considering end-to-end trajectory generation, we focus on finding a trajectory connecting two subregions obtained by the POMDP process described earlier. By iterating this process for connecting subregions (\ref{pair}), the optimal control process is able to attain the desired output with reduced computational cost. 

 \begin{figure*}
\centering
\begin{minipage}[t]{\linewidth}
\includegraphics[width=\linewidth]{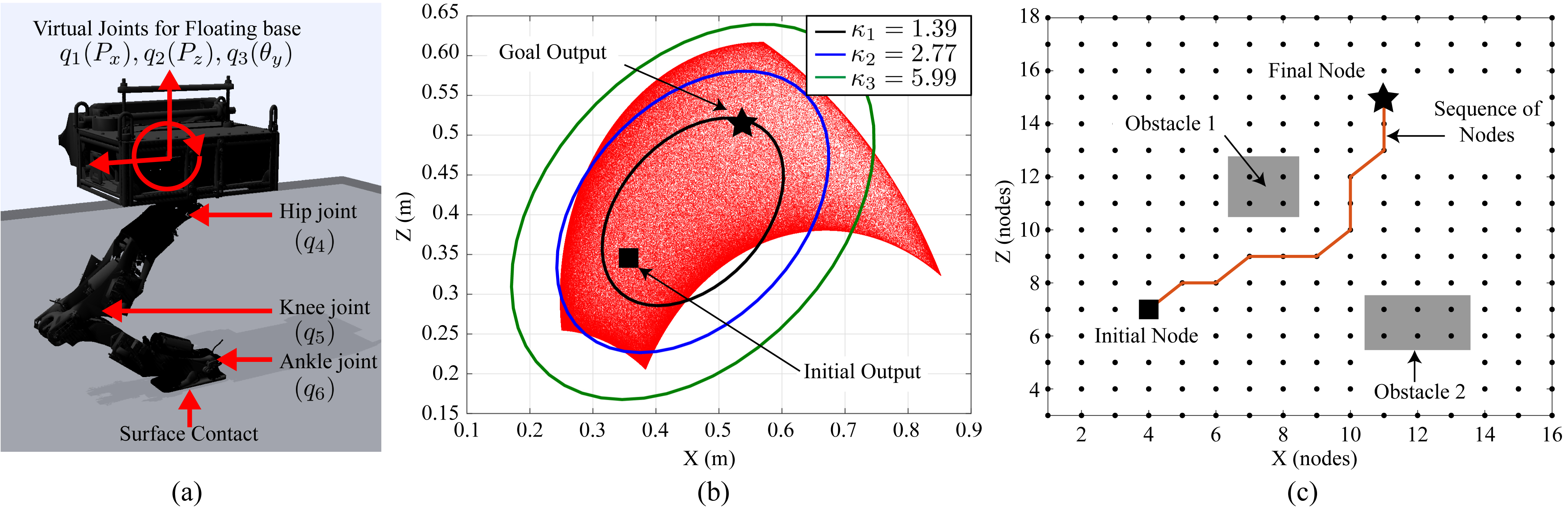}
\end{minipage}
\caption{(a) Legged robotic system with surface contact, consisting of three virtual joints and three actuated joints, (b) A set of output samples satisfying the constraints and the approximated ellipsoids with respect to $\kappa$ in (\ref{ellip}), (c) A sequence of nodes solved by POMDP. }
\label{Fig1}
\end{figure*}

\subsection{Nonlinear Programming}
In order to formulate the optimal control problem solved by NLP, a performance measure is defined in the discrete time and state space, that is,
\begin{subequations}
\begin{align*}
    &\mathcal{L}(\mathbf{U},N) \coloneqq  \sum_{k=0}^{N} (u^{\top}(t_k)\mathbf{Q}_{u}u(t_k) + F_c^{\top}(t_k)\mathbf{Q}_{c}F_{c}(t_k) \\
    &\qquad  \qquad \quad  + e_x^{\top}(t_k) \mathbf{Q}_{x} e_x(t_k) ) + e_{y}^{\top} (t_N) \mathbf{Q}_{y} e_{y}(t_N) \\
    & e_x(t_k) \coloneqq \xi(t_k) - x_0 , \quad    e_y(t_k) \coloneqq (f_y(\xi(t_k))-y_d)
\end{align*}
\end{subequations}
where $\mathbf{U} \coloneqq \{ u(t_0), u(t_1), \dots, u(t_N) \}$. In addition, $\mathbf{Q}_{x} \in \mathbb{S}_{>0}^{n_x}$ and $\mathbf{Q}_{y}\in \mathbb{S}_{>0}^{n_y}$ denote the weighting matrices for the state and output terms, respectively. $\xi(t) \in \mathbb{R}^{n_x}$ and $y_{d} \in \mathbb{R}^{n_y}$ denote the trajectory of the state and the desired goal of the output of the NLP, respectively. Based on the set of reachable states, we can formulate the NLP as follows:
\begin{equation} \label{NLP}
\begin{split}
    \min_{\mathbf{\Psi},\mathbf{U}} \quad &\mathcal{L}(\mathbf{U}, N) \\
    \mathrm{s.t.} \quad &  \xi(t_{k+1}) = f_{\mathfrak{D}}(\xi(t_k),u(t_k),F_c(t_k)) \\
    &  \xi(t_k) \in \mathcal{R}_{x}^{D}([t_0,t_k]_d^{\Delta t},x_0)\\
    &  \xi(t_{k+1}) \in \mathcal{R}_{x}^{D}([t_0,t_{k+1}]_d^{\Delta t},x_0) \\
    &  h_{x[k_x]}(\xi(t_{k+1})) \leq 0, \: \forall k_{x} \in [1, n_x^c]_{\mathbb{N}}  \\
    &  h_{xu[k_{xu}]}(\xi(t_{k}),u(t_k) \leq 0, \: \forall k_{xu} \in [1, n_{xu}^c]_\mathbb{N} \\
    &  h_{xc}(\xi(t_k), u(t_k)) \leq 0, \: u(t_k) \in \mathcal{U} .
\end{split}
\end{equation}
The proposed approach recursively executes the formulated NLP (\ref{NLP}) by changing the initial conditions, the constraints, and the goal output. Let us represent the optimization problem for connecting $\mathcal{Y}_{\star}^{(j)}$ to $\mathcal{Y}_{\star}^{(j+1)}$ as 
\begin{equation}
    (\mathbf{\Psi}^{(j)}, \mathbf{U}^{(j)}) = \mathrm{NLP} (x_0^{(j)}, y_{d}^{(j)}, t_0^{(j)}, N^{(j)})
\end{equation}
where $j \in [1, n_\pi -1]_{\mathbb{N}}$. We replace $t_0^{(j+1)}$, $x_0^{(j+1)}$ and $y_d^{(j+1)}$ with $t_N^{(j)}$, $\xi^{(j)}(t_N^{(j)})$ and $y_c^{(j+1)}$, respectively. $N^{(j+1)}$ is determined by reachability analysis using $t_0^{(j+1)}$, $x_0^{(j+1)}$, and $y_d^{(j+1)}$. In particular, we set $x_0^{(1)} = x_0$ and $y_d^{(n_{\pi} -1)} = y_{g}$. After performing the iterative NLP, we obtain the entire state and input trajectories to reach the goal output such that
\begin{subequations}
\begin{align}
    \mathbf{\Psi}^{\star} &\coloneqq \{ \mathbf{\Psi}^{(1)}, \dots, \mathbf{\Psi}^{(n_{\pi}-1)} \} \\
    \mathbf{U}^{\star} &\coloneqq \{ \mathbf{U}^{(1)},\dots, \mathbf{U}^{(n_{\pi}-1)} \} \textrm{.}
\end{align}
\end{subequations}
The solutions, $\mathbf{\Psi}^{\star}$ and $\mathbf{U}^{\star}$, enable the robotic system to reach the goal output state maintaining all constraints and contacts.  
 
\section{Numerical Simulations}

In this section, we validate the proposed approach by simulating a legged robot \textit{Draco}, which is developed for efficient and dynamic locomotion using liquid-cooled series elastic actuators \cite{kim2018investigations}. The dynamic simulation is implemented by DART \cite{lee2018dart}. We utilize two optimizers: Goldfarb for QP and IPOPT implementing a primal-dual interior point method \cite{wachter2006implementation}. The simulation is executed on a laptop with a Core i7-8650U CPU and $16.0$ GB RAM.       

\subsection{A Robotic System and Constraints}
A simulation model of \textit{Draco} consists of three virtual joints, i.e. position and orientation of its floating base ($q_1$, $q_2$, $q_3$) and three actuated joints, i.e. knee and ankle joints ($q_4$, $q_5$, $q_6$) as shown in Fig. \ref{Fig1} (a). For our simulations, the state and input constraints are defined as follows:
\begin{equation}
    \begin{split}
        &q_{j_1}^{\ominus} \leq q_{j_1} \leq q_{j_1}^{\oplus}, \textrm{  } \dot{q}_{j_1}^{\ominus} \leq \dot{q}_{j_1} \leq \dot{q}_{j_1}^{\oplus}, \textrm{  }  u_{j_1}^{\ominus} \leq u_{j_1} \leq u_{j_1}^{\oplus}\\
        &u_{j_2} = 0, \quad j_1 \in \{4,5,6\}, \quad j_2 \in \{1,2,3\} 
    \end{split}
\end{equation}
where superscripts $(.)^{\ominus}$ and $(.)^{\oplus}$ represent the lower and upper limits, respectively. The specific conditions are $q^{\ominus}_{j_4,j_5,j_6} = [ - 1.2, \:  0.5. \: -1.5]$, $q^{\oplus}_{j_4,j_5,j_6} = [ -0.2, \: 2.6, \: -0.5]$, $u^{\ominus}_{j_4,j_5,j_6} = [ -1000, \: -1000, \: -1000]$, and $u^{\oplus}_{j_4,j_5,j_6} = [ 1000, \: 1000, \: 1000]$. The position, orientation, and velocity of the foot should satisfy the kinematic constraints for the contact. We consider a surface contact with rectangular support polygon on the foot so that the friction cone constraints can be characterized as 
\begin{equation} \label{constraints}
        |f_x| \leq k_{\mu} f_z,\quad  f_z >0, \quad | \tau _y | \leq d_x f_z
\end{equation}
where $d_x$ denotes the distance between the center of the polygon and the vertex in the local frame of the foot and $k_{\mu}$ represents the friction coefficient. $F_c \coloneqq [f_x, f_z, \tau_y ]^{\top}$ is the contact wrench, which is a resultant contact force at the center of the support polygon. Based on the inequality constraints of (\ref{constraints}) and the coordinate transformation from local frame on the foot to the global frame, we can represent the friction cone constraints in the form $\mathbf{W}_{local} \mathbf{R}_{c}(q) F_c = \mathbf{W}_c(q) F_c \leq 0$ where $\mathbf{W}_{local}$ is a coefficient matrix derived from (\ref{constraints}) and $\mathbf{R}_{c}(q)$ is a rotational matrix from global to foot frames. In $\mathbf{W}_{local}$, we set the friction coefficient $k_{\mu} = 0.4$. Considering all constraints, we will control the robot's motion while maintaining the contact.  
\begin{figure}
\centering
 \includegraphics[width=\columnwidth]{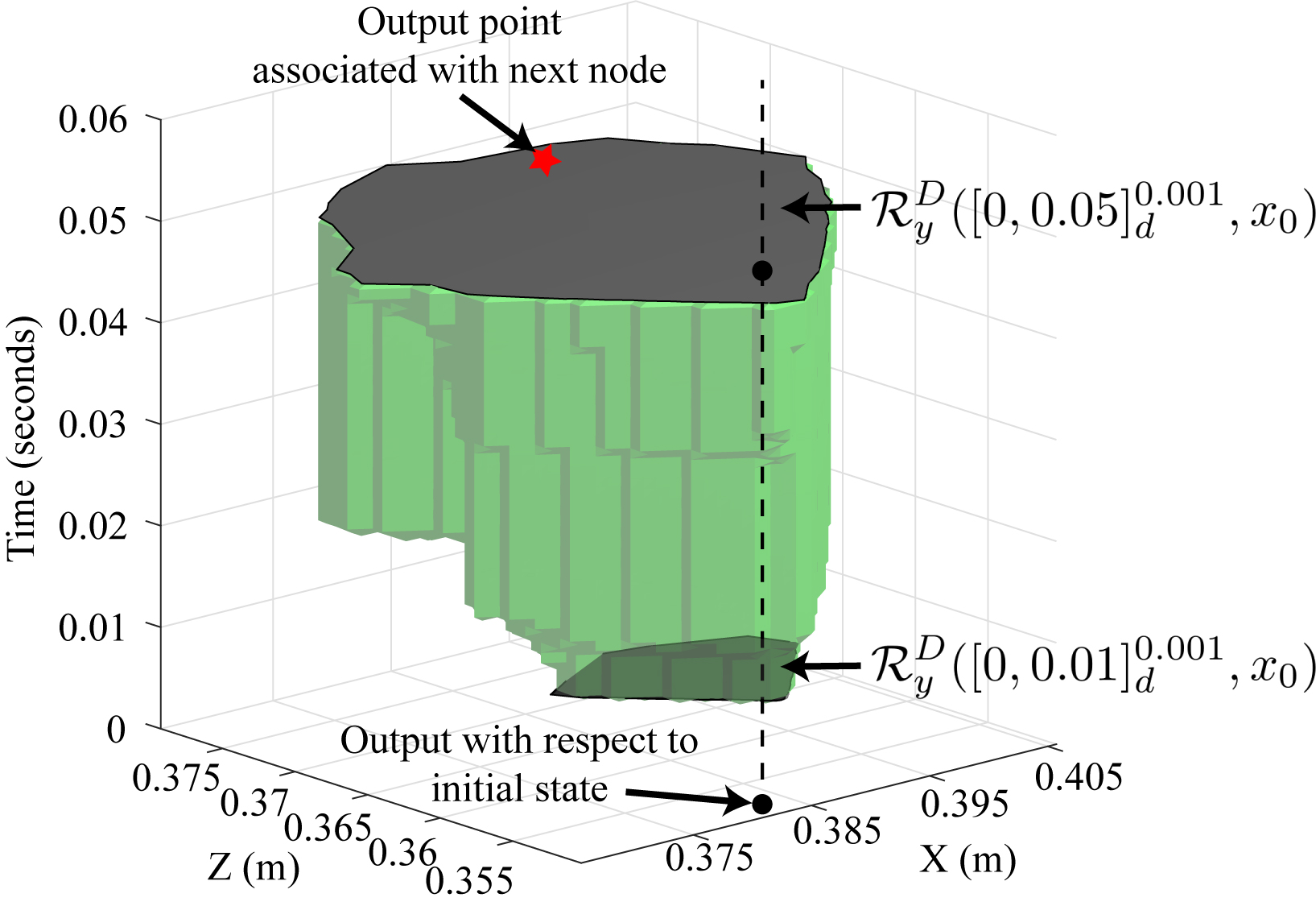}
\caption{Reachable sets over a finite time horizon $[0, 0.05]_{d}^{0.001}$ : The reachable sets are obtained by boundary-state propagation.}
\label{fig:sim_reach}
\end{figure}

\subsection{Setup for Simulation}
To start the reachability analysis, we generate $10^6$ state samples and gather the states fulfilling the constraints. The mean and covariance of the Gaussian distribution for sampling are 
\begin{equation}
    \begin{split}
        \mu_x &= [0.352, 0.384, -0.95, 2.2, -1.25, \mathbf{0}_{1\times 6}]^{\top} \\
        \mathbf{\Sigma}_{x} &= \left[\begin{array}{cc} \pi \mathbf{I}_{6\times 6} & \mathbf{0}_{6 \times 6} \\ \mathbf{O}_{6 \times 6}  & 2 \pi \mathbf{I}_{6 \times 6} \end{array} \right] \textrm{.}
    \end{split}
\end{equation}
We define a threshold for numerical convergence (\ref{state_stisfying}) with value $1.0 \times 10^{-7}$ and the maximum number of iterations is $1.0 \times 10^{6}$. After implementing the sampling-based approach described in Section \ref{sec3}, we obtain $3.47 \times 10^5$ states among $10^6$ state samples. For formulating the POMDP problem in $2$D space, we set $40$ nodes defined by $\mathbf{S} =\{s_i\}$ where $i\in [1,40]_{\mathbb{N}}$ and $8$ actions $\mathbf{A} = \{ a_{j} \}$ where $j\in [1,8]_{\mathbb{N}}$, and each action consists of moving up, down, right, left, up-right, up-left, down-right, and down-left in the grid world, respectively. We consider two static obstacles for the robot to avoid in the output space. The objective of our numerical simulation is to obtain an optimized trajectory to reach the goal output while avoiding the obstacles and fulfilling all constraints. In addition, we generate $1.0 \times 10^5$ input samples for propagating the states in the reachability analysis. 

\subsection{Simulation Results}
First, the results of our sample-based optimization is shown in Fig. \ref{Fig1}(b). The set with red dots contains the outputs associated with the states fulfilling the constraints given by the optimization process described in Section \ref{sec3}. As shown in Fig. \ref{Fig1}(b), both the initial $[ 0.384 ,\:  0.352 ]$ and goal $[ 0.51 ,\: 0.52]$ outputs are located at the interior of the feasible set. Then, we solve the POMDP problem to find a sequence of nodes, which result in $12$ of them, avoiding the obstacles as shown in Fig. \ref{Fig1}(c). After obtaining the sequence of nodes, we obtain the reachable sets as shown in Fig. \ref{fig:sim_reach}. The reachable set $\mathcal{R}_{y}^{D}([0,0.05]_{d}^{0.001},x_0)$ contains the desired output associated with the first node. Based on this result, we solve the NLP (\ref{NLP}) to find a trajectory to reach the desired output from the initial configuration.         

\begin{figure}
\centering
 \includegraphics[width=\columnwidth]{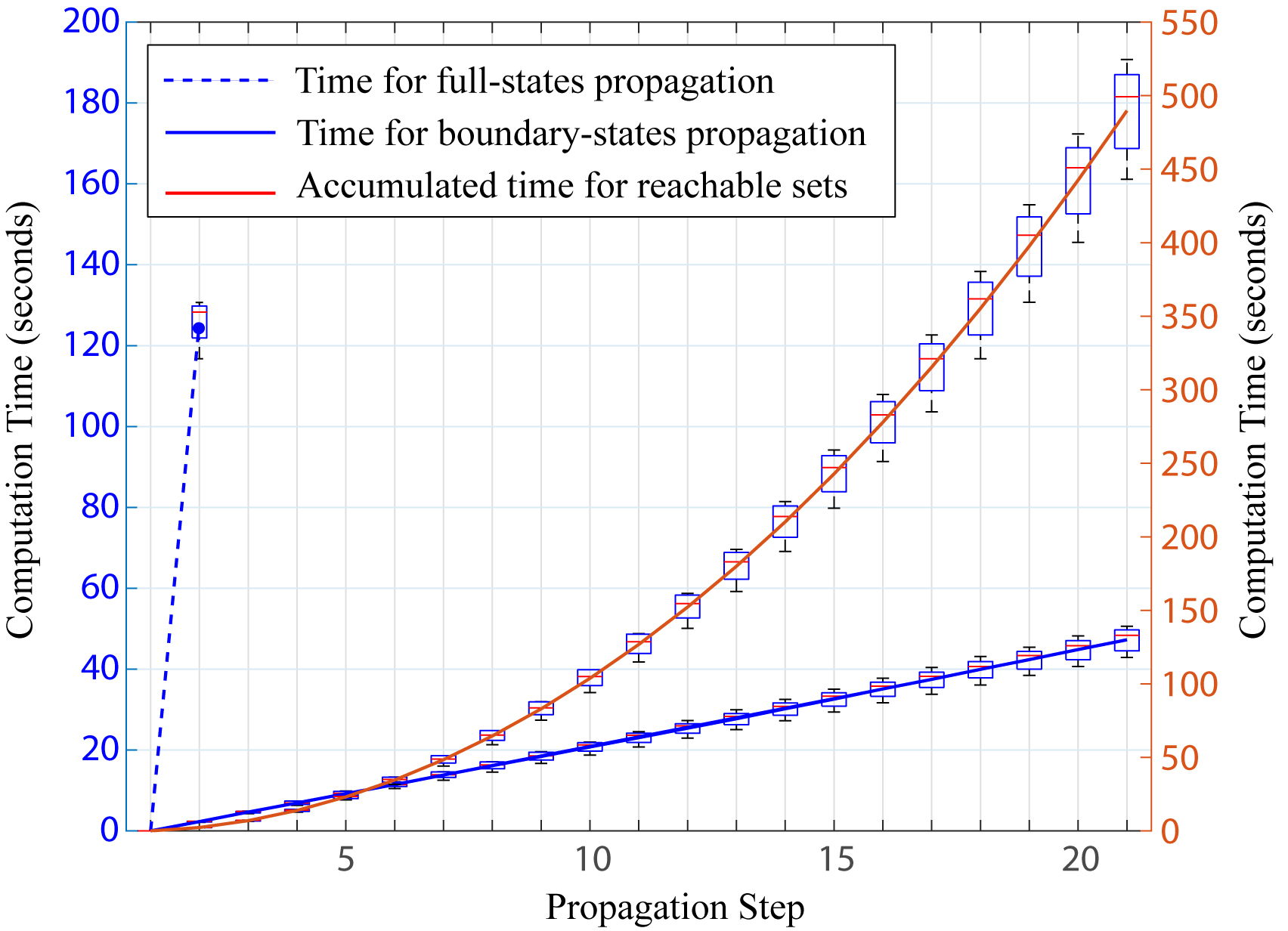}
\caption{Computational time for obtaining reachable sets with discretization step $\Delta t = 0.01$ and $10$ simulations.}
\label{fig:compute_reach}
\end{figure}

 \begin{figure*}
\centering
\begin{minipage}[t]{\linewidth}
\includegraphics[width=\linewidth]{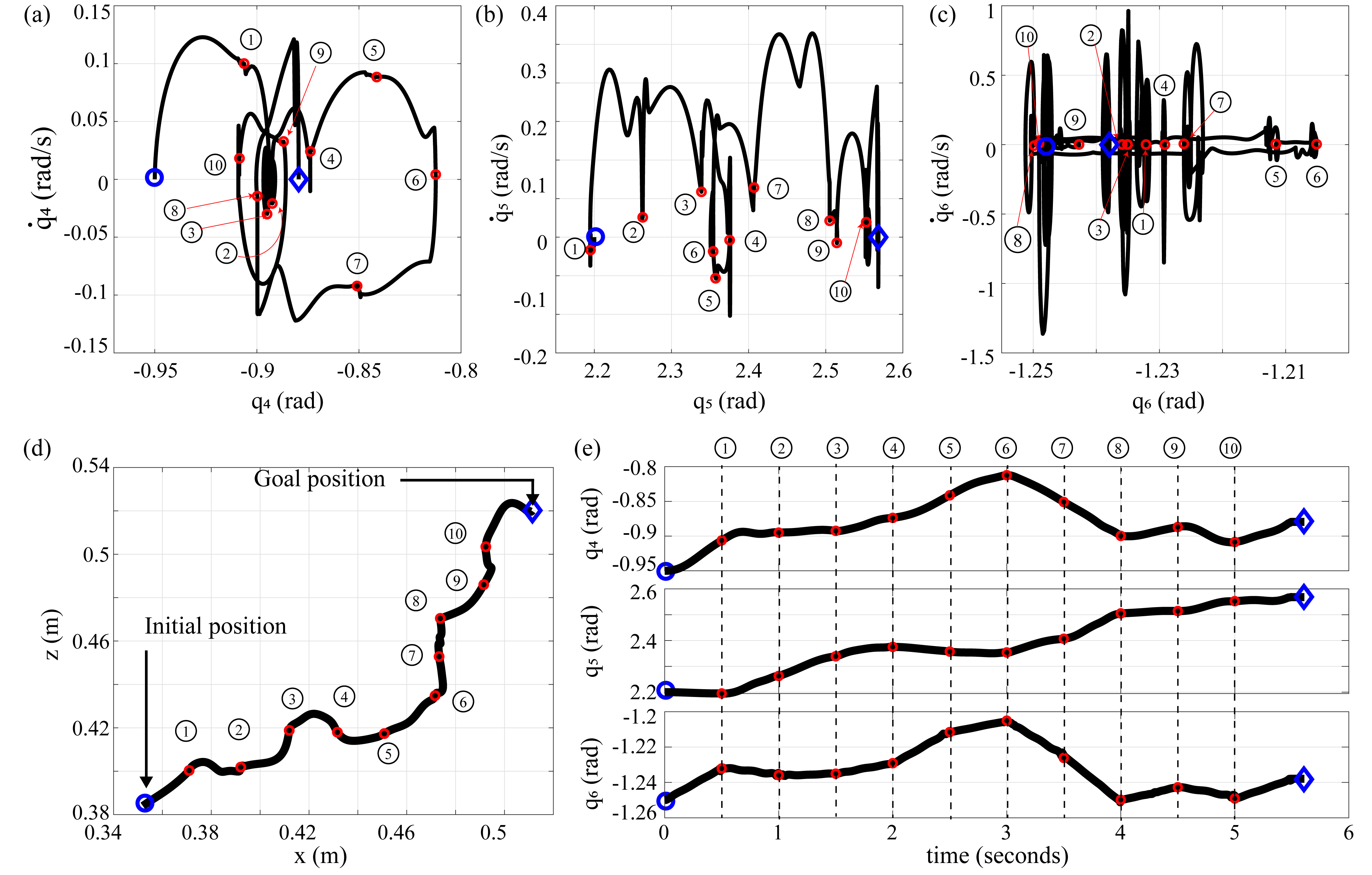}
\end{minipage}
\caption{Simulation results: (a), (b), and (c) illustrate the joint position and velocity trajectories of the actuated joints in the phase space, (d) shows the optimized trajectory in the output space, (e) shows the joint position trajectory of the actuated joints in the time domain. The blue circle and red diamond indicate the initial and final values, respectively. }
\label{Fig:sim_result}
\end{figure*}

The computational complexity is analyzed by measuring the execution time of the algorithm for computing the reachable sets. We repeat $10$ simulations to measure the computation time for both the full-state and boundary-state propagation methods and display the results in Fig. \ref{fig:compute_reach}. The algorithm cannot compute the reachable sets via the full-state propagation method for more than two steps due to the lack of memory capacity as shown in the blue dotted line in Fig \ref{fig:compute_reach}. As we predicted in the complexity analysis of Section \ref{sec_complexity}, the boundary-states propagation method significantly reduces the computational time for computing reachable sets.  

Fig. \ref{Fig:sim_result} shows the results of trajectory optimization to reach the goal position with respect to a given initial state. The optimization result includes both joint position and velocity trajectories fulfilling kinodynamic constraints, e.g. joint position, velocity, and torque limits and contact kinematic and force constraints. The optimization results for the actuated joints in the phase space have stabilizable end points which are marked as blue diamonds in Fig. \ref{Fig:sim_result}(a), (b), and (c). As shown in Fig \ref{Fig:sim_result}(d), the generated trajectories pass through the subregions in an optimized sequence obtained by solving the POMPD problem and reaching the final output goal position. 

\section{Conclusion}
This paper proposes a method to generate trajectories for complex robotic systems considering contact constraints. We utilize NLP to solve the optimal control problem. Our approach focuses on efficiently solving the NLP problem so that we can scale the method to many types of complex robotic systems. We devise a new approach to obtain discrete-time reachable sets for trajectory generation and solve the nonlinear optimization problem. Although the computational cost is significantly reduced, it is still challenging to employ this approach to real-time control. Therefore, in the near future, we will investigate ways to combine this approach with feedback controllers and extend the proposed method for hybrid dynamical systems, such as biped humanoid robots or dual arm manipulators. 

\begin{ack}                               % Place acknowledgements
The authors would like to thank the members of the Human Centered Robotics Laboratory at The University of Texas at Austin for their great help and support. This work was supported by an NSF Grant\# 1724360 and partially supported by an ONR Grant\# N000141512507.
\end{ack}

\bibliographystyle{plain}        % Include this if you use bibtex 
\bibliography{autosam}           % and a bib file to produce the 
                                 % bibliography (preferred). The
                                 % correct style is generated by
                                 % Elsevier at the time of printing.

%\begin{thebibliography}{99}     % Otherwise use the  
                                 % thebibliography environment.
                                 % Insert the full references here.
                                 % See a recent issue of Automatica 
                                 % for the style.
%  \bibitem[Heritage, 1992]{Heritage:92}
%     (1992) {\it The American Heritage. 
%     Dictionary of the American Language.}
%     Houghton Mifflin Company.
%  \bibitem[Able, 1956]{Abl:56}
%     B.~C.~Able (1956). Nucleic acid content of macroscope. 
%     {\it Nature 2}, 7--9. 
%  \bibitem[Able {\em et al.}, 1954]{AbTaRu:54}   
%     B.~C. Able, R.~A. Tagg, and M.~Rush (1954).
%     Enzyme-catalyzed cellular transanimations.
%     In A.~F.~Round, editor, 
%     {\it Advances in Enzymology Vol. 2} (125--247). 
%     New York, Academic Press.
%  \bibitem[R.~Keohane, 1958]{Keo:58}
%     R.~Keohane (1958).
%     {\it Power and Interdependence: 
%     World Politics in Transition.}
%     Boston, Little, Brown \& Co.
%  \bibitem[Powers, 1985]{Pow:85}
%     T.~Powers (1985).
%     Is there a way out?
%     {\it Harpers, June 1985}, 35--47.

%\end{thebibliography}

\appendix

\end{document}